\documentclass{article}
\usepackage[utf8]{inputenc}
\usepackage{authblk}
\usepackage{basicreq}
\usepackage{mathnotations}
\usepackage{natbib}
\usepackage{appendix}

\begin{document}

\title{On the Minimax Regret for Linear Bandits in a wide variety of Action Spaces}
\author[1]{Debangshu Banerjee}
\author[2]{Aditya Gopalan}
\affil[1,2]{ Department of Electrical and Communication Engineering, Indian Institute of Science, India}
\date{November 2022}
\maketitle

\begin{abstract}
    As noted in the works of \cite{lattimore2020bandit}, it has been mentioned that it is an open problem to characterize the minimax regret of linear bandits in a wide variety of action spaces. In this article we present an optimal regret lower bound for a wide class of convex action spaces.
\end{abstract}

\section{Introduction}
Minimax regret bounds in bandit environments are a well studied problem and results typically are limited to a particular action set, namely the $l_1$ and $l_\infty$ balls in $\mathbb{R}^d$. We include in this article that display that the methods introduced by \cite{lattimore2020bandit} in Chapter 24 can indeed be generalized to a wide variety of action spaces, namely to any $l_p$ ball where $p$ is in the range $(1,\infty)$. 

\section{Key Result}
Note that the result we include in \ref{thm:main}, is optimal in the bandit setting, in sense that algorithms like LinUCB achieve this.  
\begin{theorem}
\label{thm:main}
Let $\cX$ be the $L^p$ ball defined as $\cX = \{ x \in \Real^d \; : \norm{x}_p \leq c \}$, where $1 < p < \infty$. Assume $d \leq (2cn^2)^\frac{p}{2}$. Then there exists a parameter $\theta \in \Real^d$ with $\norm{\theta}_p^p = \frac{1}{(c4\sqrt{3})^p}\frac{d^2}{n^{\frac{p}{2}}}$ such that $\cR_n(\theta) \geq \frac{d\sqrt{n}}{16\sqrt{3}}$.  
\end{theorem}
\begin{proof}

We chose $\theta \in \{ +\Delta, -\Delta\}^d$ and note that the regret, defined as, $\cR_n(\theta)$, is
\begin{align}
\begin{split}
    \cR_n(\theta) = \sum_{t=1}^n x^*{^\top}\theta - x_t^\top \theta\\
    = \sum_{t=1}^n \sum_{i=1}^d x^*_i\theta_i - x_{t i}\theta_i \\
    = \Delta \sum_{t=1}^n \sum_{i=1}^d \frac{c}{d^\frac{1}{p}} - x_{ti}sign(\theta_i)\\
    \geq \frac{\Delta d^\frac{1}{p}}{2c}\sum_{t=1}^n\sum_{i=1}^d \bigg(\frac{c}{d^\frac{1}{p}} - x_{ti}sign(\theta_i)\bigg)^2\; ,
\end{split}
\end{align}
where the third equality follows from Lemma \ref{lemma:optimal} and the last inequality follows from Lemma \ref{lemma:trick}. The remainder of the proof follows the same idea as that presented in the proof of the Unit Ball in Section 24.2 of \cite{lattimore2020bandit}. We present it here for the sake of completeness.
We define a stopping time $\tau_i = n \wedge \min{\{t \; : \; \sum_{s=1}^t x_{si}^2 \geq \frac{nc^2}{d^\frac{2}{p}}\}}$. Thus 
\begin{align*}
    \cR_n(\theta) \geq \frac{\Delta d^\frac{1}{p}}{2c}\sum_{i=1}^d\sum_{t=1}^{\tau_i} \bigg(\frac{c}{d^\frac{1}{p}} - x_{ti}sign(\theta_i)\bigg)^2\;.
\end{align*}
Define a Random Variable $U_i(\sigma) = \sum_{t=1}^{\tau_i} \bigg(\frac{c}{d^\frac{1}{p}} - x_{ti}\sigma\bigg)^2$ where $\sigma \in \{+1, -1\}$ and note that
\begin{align}
\label{eq:upper_bound}
U_i(1) =  \sum_{t=1}^{\tau_i} \bigg(\frac{c}{d^\frac{1}{p}} - x_{ti}\bigg)^2 \leq 2\sum_{t=1}^{\tau_i} \frac{c^2}{d^\frac{2}{p}} + 2\sum_{t=1}^{\tau_i}x_{ti}^2 \leq  \frac{4nc^2}{d^\frac{2}{p}} + 2\;, \end{align}
where the last inequality follows from the definition of $\tau_i$. 

Now we fix an $i$, and make a perturbed version of $\theta'$, which is the same as $\theta$ except in the $i^{th}$ position where $\theta'_i = - \theta_i$. Thus, applying Pinsker's inequality, 
\begin{align}
\mathbb{E}_\theta[U_i(1)] \geq \mathbb{E}_{\theta'}[U_i(1)] - \bigg(\frac{4nc^2}{d^\frac{2}{p}} + 2\bigg) \sqrt{\frac{1}{2}\mathbb{KL}(\mathbb{P}_\theta || \mathbb{P}_{\theta'})}    
\end{align}
\begin{align*}
    \geq \mathbb{E}_{\theta'}[U_i(1)] - \frac{\Delta}{2}\bigg(\frac{4nc^2}{d^\frac{2}{p}} + 2\bigg)\sqrt{\sum_{t=1}^{\tau_i}x_{ti}^2}
\end{align*}
\begin{align*}
    \geq \mathbb{E}_{\theta'}[U_i(1)] - \frac{\Delta}{2}\bigg(\frac{4nc^2}{d^\frac{2}{p}} + 2\bigg)\sqrt{\frac{nc^2}{d^\frac{2}{p}}+1}
\end{align*}
\begin{align*}
  \geq  \mathbb{E}_{\theta'}[U_i(1)] - \frac{4\sqrt{3}n\Delta c^2}{d^\frac{2}{p}}\sqrt{\frac{nc^2}{d^\frac{2}{p}}}\;,
\end{align*}
where the last inequality follows under the assumption $d \leq (2n c^2)^\frac{p}{2}$.
Thus 
\begin{align*}
    \mathbb{E}_\theta[U_i(1)] + \mathbb{E}_{\theta'}[U_i(-1)] \geq
    \mathbb{E}_{\theta'}[U_i(1)) + U_i(-1)] - \frac{4\sqrt{3}n\Delta c^2}{d^\frac{2}{p}}\sqrt{\frac{nc^2}{d^\frac{2}{p}}}
\end{align*}
\begin{align*}
\begin{split}
    = 2\mathbb{E}_{\theta'}\bigg[\frac{\tau_i c^2}{d^\frac{2}{p}} + \sum_{t=1}^{\tau_i}x_{ti}^2\bigg] - \frac{4\sqrt{3}n\Delta c^2}{d^\frac{2}{p}}\sqrt{\frac{nc^2}{d^\frac{2}{p}}} 
    \geq \frac{nc^2}{d^\frac{2}{p}}\;,
\end{split}    
\end{align*}
where the last inequality follows from the definition of $\tau_i$ and setting the value of $\Delta$ as $\frac{1}{4\sqrt{3}}\sqrt{\frac{d^\frac{2}{p}}{nc^2}}$.
Using an average hammering trick
\begin{align*}
 \sum_{\theta \in \{\pm \Delta\}^d} \cR_n(\theta) \geq  \frac{\Delta d^\frac{1}{p}}{2c}\sum_{i=1}^d \sum_{\theta \in \{\pm \Delta\}^d}\mathbb{E}_\theta[U_i(sign(\theta_i)] 
\end{align*}
\begin{align*}
\begin{split}
   = \frac{\Delta d^\frac{1}{p}}{2c}\sum_{i=1}^d \sum_{\theta_{-i} \in \{\pm \Delta\}^{d-1}}\sum_{\theta_i \in \{\pm \Delta\}}\mathbb{E}_\theta[U_i(sign(\theta_i)]\\
   \geq \frac{\Delta d^\frac{1}{p}}{2c}\sum_{i=1}^d \sum_{\theta_{-i} \in \{\pm \Delta\}^{d-1}}\frac{nc^2}{d^\frac{2}{p}}
   =2^{d-2}\Delta n c d^{1-\frac{1}{p}}.
\end{split}
\end{align*}
Hence there exists a $\theta$ in $\{\pm\Delta\}^d$, such that
\begin{align*}
    \cR_n(\theta) \geq \frac{nc\Delta d^{1-\frac{1}{p}}}{4} = \frac{d\sqrt{n}}{16\sqrt{3}}.
\end{align*}
\end{proof}

\begin{remark}
The results in \ref{thm:main} are interesting because with regards to the dimensionality $d$ and time horizon $n$ dependence it is exact. 
\end{remark}

\begin{remark}
This result also shows that the minimum eigen value of the design matrix and regret are fundamentally different quantities \cite{banerjee2022exploration}. For example note that for $l_p$ balls for $p > 2$, the minimum eigen value can grow at a rate lower than $\Omega(\sqrt{n}$, whereas, the minimax regert remains bounded as $\Omega(\sqrt{n}$.
\end{remark}

\section{Conclusion}
We expect that similar results can hold for general convex bodies and not just for $l_p$ balls.

\bibliography{ref}

\appendix

\section{Appendix}

\subsection{Technical Lemmas}
\begin{lemma}[Optimal Reward in $L^p$ Ball]
\label{lemma:optimal}
Let $\cX$ be the $L^p$ ball defined as $\cX = \{ x \in \Real^d \; : \norm{x}_p \leq c \}$. We compute the optimal reward for the linear bandit model 
\begin{align}
\label{problem:optimal_reward}
\max x^\top \theta \;\;\;
s.t. \;\;\; x \in \cX    
\end{align}
The solution to the optimization problem \ref{problem:optimal_reward} is $\frac{1}{\Big(\sum_{i=1}^d |\theta_i|^\frac{p}{p-1}\Big)^\frac{1}{p}} \sum_{i=1}^d c |\theta_i|^{\frac{p}{p-1}} $
\end{lemma}
\begin{proof}
Note that the solution $x^*$ satisfies the following relation for any $i \in [d]$
\begin{align}
    x^*_i = \Big(\frac{|\theta_i|}{\lambda}\Big)^{\frac{1}{p-1}} sign(\theta_i)\;,
\end{align}
where $\lambda \geq 0 $ is the Lagrangian variable.
Solving for $\lambda$ using the constraint equation of the problem with now equality instead of inequality. (Because the optimal solution lies at the boundary)
\begin{align}
    \lambda = \Big(\frac{\sum_{i=1}^d |\theta_i|^\frac{p}{p-1}}{c^p}\Big)^\frac{p-1}{p}\;.
\end{align}
Now solving for $x^*{^\top} \theta = \sum_{i=1}^d x^*_i \theta_i$ gives the result.
\end{proof}

\begin{lemma}
\label{lemma:trick}
\begin{align}
    \sum_{i=1}^d \frac{c}{d^\frac{1}{p}} - x_{ti}sign(\theta_i) \geq \frac{d^\frac{1}{p}}{2c}\sum_{i=1}^d \bigg(\frac{c}{d^\frac{1}{p}} - x_{ti}sign(\theta_i)\bigg)^2
\end{align}
\end{lemma}
\begin{proof}
\begin{align}
\begin{split}
\sum_{i=1}^d \bigg(\frac{c}{d^\frac{1}{p}} - x_{ti}sign(\theta_i)\bigg)^2 \\
= c^2d^{1 - \frac{2}{p}} - 2\sum_{i=1}^d\frac{c}{d^\frac{1}{p}}x_{ti}sign(\theta_i) + \sum_{i=1}^d x_{ti}^2 \\
\leq 2c^2d^{1 - \frac{2}{p}} - 2\sum_{i=1}^d\frac{c}{d^\frac{1}{p}}x_{ti}sign(\theta_i) \\
= \frac{2c}{d^\frac{1}{p}}\sum_{i=1}^d \frac{c}{d^\frac{1}{p}} - x_{ti}sign(\theta_i)\;,
\end{split}    
\end{align}
where the inequality follows from Lemma \ref{lemma:norm_equivalence}. Rearranging gives the lemma.
\end{proof}

\begin{lemma}
\label{lemma:norm_equivalence}
If $\norm{x}_p \leq c$, then $\norm{x}_2^2 \leq c^2d^{1 - \frac{2}{p}} $
\end{lemma}
\begin{proof}
\begin{align}
\sum_{i=1}^d x_{i}^2 \leq (\sum_{i=1}^d |x_{i}|^p)^\frac{2}{p}d^{1 - \frac{2}{p}} \leq c^2d^{1 - \frac{2}{p}}\;,     
\end{align}
where the first inequality follows from Holder's inequality and the second inequality follows from the hypothesis.
\end{proof}

\end{document}